\PassOptionsToPackage{hypertexnames=false}{hyperref}
\documentclass{article}
\usepackage{iclr2027_conference,times}

\usepackage{amsmath,amssymb,graphicx,url}
\usepackage{jmlrutils}
\usepackage{hyperref}
\hypersetup{hidelinks}
\iclrfinalcopy

\usepackage[utf8]{inputenc}
\usepackage[T1]{fontenc}

\usepackage{booktabs}
\usepackage{microtype}

\usepackage{enumitem}
\setlist[itemize]{leftmargin=*}
\usepackage{placeins}
\usepackage{multirow}
\usepackage{float}
\usepackage{cleveref}

\newcommand{\E}{\mathbb{E}}
\newcommand{\R}{\mathbb{R}}
\newcommand{\Prob}{\mathbb{P}}
\newcommand{\bw}{\mathbf{w}}
\newcommand{\bm}{\mathbf{m}}
\newcommand{\bv}{\mathbf{v}}
\newcommand{\bg}{\mathbf{g}}
\newcommand{\bx}{\mathbf{x}}

\newcommand{\smallO}{o}
  
\newcommand{\msd}[2]{$#1$\,{\tiny$(#2)$}}
\newcommand{\mone}[1]{$#1$\,{\tiny\phantom{$(0.00)$}}}
\DeclareRobustCommand{\methodname}[1]{#1}
\newcommand{\uprtnmethod}{Update-RTN}
\newcommand{\upsrmethod}{Update-SR}
\newcommand{\upsrrecipe}{\methodname{ZIP-SR}}
\newcommand{\nzedenrecipe}{\methodname{ZE-EDEN}}

\theorembodyfont{\normalfont}

\theorembodyfont{\normalfont\itshape}
\crefname{assumption}{assumption}{assumptions}
\Crefname{assumption}{Assumption}{Assumptions}
\crefname{theorem}{proposition}{propositions}
\Crefname{theorem}{Proposition}{Propositions}
\crefname{definition}{definition}{definitions}
\Crefname{definition}{Definition}{Definitions}
\Crefname{appendix}{appendix}{appendices}
\Crefname{appendix}{Appendix}{Appendices}

\renewenvironment{proof}[1][\proofname]
{%
    \par\noindent{\bfseries\upshape #1\ }\ignorespaces
}
{%
    \jmlrQED
}

\title{Rounding in Preconditioner Space: Redesigning 4-bit AdamW Optimizer-State Quantization}
\author{%
\parbox[t]{\dimexpr\textwidth-2\tabcolsep\relax}{\centering
Hanyang Li\textsuperscript{1}\thanks{Work done during Hanyang Li's internship at Nubank.}\quad
Shao Tang\textsuperscript{2}\thanks{Correspondence to: Shao Tang \href{mailto:tang.shao@nubank.com.br}{\nolinkurl{tang.shao@nubank.com.br}},\newline Aman Gupta \href{mailto:aman.gupta@nubank.com.br}{\nolinkurl{aman.gupta@nubank.com.br}}.}\quad
Daniel Thomas Braithwaite\textsuperscript{2}\\[3pt]
Gregory Dexter\textsuperscript{2}\quad
Leonardo Neves\textsuperscript{2}\quad
Aman Gupta\textsuperscript{2}\footnotemark[2]\quad
Hiroto Udagawa\textsuperscript{2}\\[3pt]
Abhishek Shivanna\textsuperscript{2}\quad
Daniel Silva\textsuperscript{2}\quad
Rohan Ramanath\textsuperscript{2}\\[6pt]
{\normalfont\textsuperscript{1}University of California, Berkeley\quad
\textsuperscript{2}Nubank}%
}}

\begin{document}

\maketitle
\fancyhead{}
\renewcommand{\headrulewidth}{0pt}

\begin{abstract}
Quantizing AdamW's optimizer states reduces persistent storage, but quantization errors propagate through the moment recurrences and perturb subsequent adaptive updates. We redesign 4-bit optimizer-state quantization for AdamW from the perspective of \emph{rounding space}: the coordinate in which a quantizer chooses between adjacent reconstruction levels. For the second moment, a local analysis of the quantization cell adjacent to zero shows that small mean state error need not imply small mean preconditioner error at the next step. A one-dimensional quadratic construction further shows qualitatively different optimization dynamics under state-space and preconditioner-space rounding. These results motivate Zero-Inclusive Preconditioner-space Stochastic Rounding (\textbf{ZIP-SR}), which retains zero in the second-moment codebook and computes stochastic-rounding probabilities in preconditioner space. As a complementary route, Zero-Excluding EDEN calibration (\textbf{ZE-EDEN}) uses a zero-excluding second-moment codebook and rescales the quantized second-moment block to mitigate the preconditioner distortion caused by the positive quantization floor. Both configurations use 4-bit NormalFloat (NF4) for the first moment, with targeted stochastic rounding of the LM-head first moment during the final 10\% of training. Across GPT- and Llama-style pretraining experiments ranging from \textbf{130M} to \textbf{2.7B} parameters, both methods reduce TorchAO 4-bit AdamW's mean validation-loss gap to 32-bit AdamW at every evaluated model size, with the largest reported gap reduction reaching \textbf{70\%}. In full-parameter supervised fine-tuning, both recipes achieve lower validation loss than TorchAO while remaining close to 32-bit AdamW on downstream tasks. 

\end{abstract}

\section{Introduction}

Adam~\citep{kingma2015adam} and AdamW~\citep{loshchilov2019decoupled} are widely used optimizers for large-scale neural network training. Their FP32 first- and second-moment buffers require 8 bytes per parameter, adding substantial memory overhead and motivating research on optimizer-state compression. Block-wise dynamic quantization enabled practical 8-bit optimizers, implemented in bitsandbytes~\citep{dettmers2022bit} and TorchAO~\citep{or2025torchao}. \citet{li2023memory} extended optimizer-state quantization to 4 bits and identified the \textit{zero-point failure}, in which mapping a small positive second-moment value to zero produces an excessively large update. They showed that excluding zero from the second-moment codebook prevents this failure. TorchAO's 4-bit AdamW implements this strictly positive codebook. Subsequent studies likewise exclude zero from the second-moment codebook, using either the same linear grid or alternative positive codebooks~\citep{xu2025solo, topollai2026state}.

Our key insight is that the rounding rule is critical to whether zero-point failure occurs with a zero-inclusive codebook.
Conventional rounding chooses between neighboring levels by their values in second-moment state space, but AdamW transforms the second moment nonlinearly before it scales the update. As a result, rounding that is accurate, or unbiased, in state space can still distort the resulting update.
By defining rounding decisions in preconditioner space, we retain the same zero-inclusive codebook while reducing the probability of harmful zero selections. 
More broadly, optimizer-state quantizers should account for {how stored states affect optimization, not solely for reconstruction error in state space}.

We make two contributions.
\begin{itemize}
    \item First, we characterize how the rounding coordinate affects near-zero second-moment quantization at fixed reconstruction levels. A local zero-cell analysis and a scalar-quadratic construction show that it affects both preconditioner error and optimization dynamics. This motivates {\upsrmethod{}}, which assigns stochastic-rounding (SR) probabilities using the preconditioners induced by the candidate reconstruction levels.
    \item Second, we develop two 4-bit AdamW configurations. Zero-Inclusive Preconditioner-space Stochastic Rounding (\textbf{\upsrrecipe{}}) combines a zero-inclusive second-moment codebook with \upsrmethod{}, whereas Zero-Excluding EDEN calibration (\textbf{\nzedenrecipe{}}) combines a zero-excluding codebook with EDEN block calibration~\citep{vargaftik2022eden}. 
    Both use NF4 for the first moment and targeted LM-head stochastic rounding motivated by matched checkpoint interventions.
    At the same persistent moment-storage cost as TorchAO 4-bit AdamW, both configurations reduce its mean validation-loss gap to 32-bit AdamW across the evaluated pretraining settings and attain lower mean SFT validation loss.
\end{itemize}

PyTorch implementations of \upsrrecipe{} and \nzedenrecipe{}, along with example training scripts, are available at \url{https://github.com/nubank/adamw4bit}.

\section{Preliminaries}
\label{sec:prelim-related}

\subsection{Terminology and Notation}
We write $Q_{\phi,\rho,\mathcal C}$ for the decoded quantizer with rounding
coordinate $\phi$, rounding rule $\rho$, and decoded codebook $\mathcal C$.\footnote{Normalized formats, decoded codebooks, and rounding rules are defined in \Cref{app:quantization-background}.}
Bare $Q$ abbreviates a quantizer fixed by the surrounding context.
For $t \ge 1$ with gradient $\bg_t$, AdamW~\citep{loshchilov2019decoupled} forms the first and second-moments by
\[
    \bm_t=\beta_1\bm_{t-1}+(1-\beta_1)\bg_t,
    \qquad
    \bv_t=\beta_2\bv_{t-1}+(1-\beta_2)\bg_t^2
\]
and then updates
\[
    \bw_t=(1-\eta_t\lambda)\bw_{t-1}-\eta_t \,\frac{\bm_t/(1-\beta_1^t)}{\sqrt{\bv_t/(1-\beta_2^t)}+\epsilon},
\]
with learning rate $\eta_t$, weight-decay $\lambda$, moment decay coefficients $\beta_1, \beta_2$, and denominator stabilizer $\epsilon>0$; squares, square roots, and divisions are applied elementwise.
Optimizer-state quantization is applied when the moments are written back to the buffers retained between optimizer
steps. %
Let $\tilde \bm_t$ and $\tilde \bv_t$ be the working moments immediately before
storage. The working moments are computed from the previously decoded states as
\[
    \tilde \bm_t=\beta_1Q(\tilde \bm_{t-1})+(1-\beta_1)\bg_t, \qquad
    \tilde \bv_t=\beta_2Q(\tilde \bv_{t-1})+(1-\beta_2)\bg_t^2.
\]
The parameter update uses the working moments $\tilde{\bm}_t$ and $\tilde{\bv}_t$ before their quantized representations are stored for the next step.

\subsection{Related Work}
\textbf{Low-precision optimizer states.}
Block-wise dynamic quantization enabled practical 8-bit Adam-state
storage~\citep{dettmers2022bit}, and \citet{li2023memory}
extended this approach to 4-bit states while identifying the
second-moment zero-point failure. Subsequent work has explored paired
geometric encodings~\citep{tian2025piquant}, logarithmic-coordinate stochastic rounding and moment-specific precision~\citep{xu2025solo}, and
spatio-temporal precision allocation~\citep{liu2026stquant}. COAT
combines FP8 optimizer-state storage with low-precision activation
training~\citep{xi2025coat}; FlashOptim applies moment-specific
companding before 8-bit quantization, including a square-root transform
for the second moment~\citep{ortiz2026flashoptim}; and Full-Stack FP4
uses transformed NVFP4 representations for AdamW states within a
broader low-precision training system~\citep{ding2026fullstackfp4}.

\textbf{Dynamics of quantized optimizer states.}
Recent studies examine how quantization errors feed back through the
recurrent exponential moving averages. SOLO~\citep{xu2025solo} analyzes signal swamping
in unsigned states and variance amplification in signed
states. \citet{topollai2026state} study state staleness and periodic
resets, while BReD targets persistent drift caused by biased recurrent
rounding~\citep{zhan2026bred}. \citet{tang2026adaptive} establish convergence guarantees for adaptive optimizers under floating-point quantization.

\textbf{Complementary memory-efficient optimizers.}
Optimizer memory can also be reduced by modifying the structure of the
stored statistics or the update rule. Adafactor factorizes
second-moment estimates~\citep{shazeer2018adafactor}, SM3 shares
structured adaptive accumulators~\citep{anil2019memory}, and Adam-mini
shares adaptive learning rates within parameter
blocks~\citep{zhang2025adammini}; Q-Adam-mini combines the latter
structure with 8-bit quantization~\citep{han2025qadammini}. Gefen shares
second-moment estimates and quantizes the first moment using a learned
codebook~\citep{benedek2026gefen}, while GaLore and APOLLO maintain
optimizer information in lower-dimensional auxiliary
spaces~\citep{zhao2024galore,zhu2025apollo}. Low-bit state
representations have also been developed for other preconditioners,
including 4-bit Shampoo~\citep{wang2024shampoo} and quantized Muon~\citep{gupta2025muonquant, wu2026achieving, su2026muonq}.

\section{Second-Moment Quantization}
\label{sec:second-moment-quantization}

Near zero, second-moment quantization presents two distinct distortions. Selecting zero can produce an excessively large subsequent preconditioner, whereas excluding zero imposes a positive floor that can suppress large preconditioner entries.
We study these cases separately. ZIP-SR retains zero and changes the rounding coordinate while holding the decoded levels fixed. ZE-EDEN follows the established zero-excluding approach~\citep{li2023memory,xu2025solo,topollai2026state} and calibrates the decoded block scale to mitigate the distortion introduced by its positive floor.

\subsection{ZIP-SR: Zero-Inclusive Preconditioner-space Stochastic Rounding}
\label{subsec:zero-inclusive}

To isolate the effect of the rounding configuration from that of the codebook, we fix the representable levels. We first compare next-step preconditioner bias and derive a rule-independent criterion for how one stored value affects the next-step state and preconditioner. Then, we specialize this criterion to round-to-nearest (RTN) and SR in both state- and preconditioner-space before giving a scalar-quadratic construction that illustrates their different optimization dynamics.

Under the update-before-storage convention in
\Cref{sec:prelim-related}, quantizing the working second moment $\tilde v_t$ affects the next update, after entering the second-moment recurrence together with $g^2_{t+1}$. For $v\geq 0$, define the current step and next-step preconditioner maps by
\[
    h_t(v):=\frac{1}{\sqrt{v/(1-\beta_2^t)}+\epsilon},
    \qquad
    r_{t+1}(v):=h_{t+1}\bigl(\beta_2 v+(1-\beta_2)g_{t+1}^2\bigr).
\]
We compare storing $Q(\tilde v_t)$ with storing the unquantized
working value $\tilde v_t$ through the mean state and
preconditioner errors
\[
    \Delta_v:=\E[Q(\tilde v_t)]-\tilde v_t,
    \qquad
    \Delta_r:=\E[r_{t+1}(Q(\tilde v_t))]
              -r_{t+1}(\tilde v_t).
\]
All expectations are conditional on the realized values of
$\tilde v_t$ and $g_{t+1}$ and are taken only over the step-$t$
rounding randomness. %

\paragraph{State-space versus preconditioner-space rounding.}
For adjacent levels $0\leq a<b$ in a nonnegative decoded codebook $\mathcal C_v$, following the notation in \Cref{app:quantization-rounding}, state-space rounding uses the coordinate $\phi(v)=v$ and preconditioner-space rounding uses $\phi(v)=h_t(v)$. 
We use $h_t$ as the rounding coordinate at storage time while $r_{t+1}$ is used only for analysis because it depends on the next gradient $g_{t+1}$. 
For $\rho \in \{\mathrm{RTN}, \mathrm{SR}\}$, we write State-$\rho$ for $Q_{\mathrm{id}, \rho, \mathcal C_v}$ and Update-$\rho$ for $Q_{h_t, \rho, \mathcal C_v}$. RTN selects the endpoint whose value in the chosen coordinate is closest to $\phi(v)$. For $v \in (a,b)$, SR selects the lower level $a$ with probability
\[
    \Prob(Q_{\phi, \mathrm{SR}, \mathcal{C}_v}(v) = a)
    =\left\{
    \begin{array}{cl}
        \frac{b-v}{b-a}, & \text{if }\phi = \mathrm{id},\\[3pt]
        \frac{h_t(v) - h_t(b)}{h_t(a) - h_t(b)}, & \text{if }\phi = h_t.
    \end{array}\right.
\]
Thus State-SR preserves the mean state value, whereas Update-SR preserves the mean of the current step preconditioner, i.e., $\E[h_t(Q_{h_t,\mathrm{SR},\mathcal C_{v}}(\tilde v_t))]=h_t(\tilde v_t)$. 
The next proposition compares their means after one step of moment recurrence. Proofs are given in \Cref{app:one-step-error-proofs}.

\pagebreak[0]
\begin{samepage}
\begin{proposition}[Next-step preconditioner bias]
\phantomsection\label{prop:zero-cell-orders}
Fix $t\geq1$ and $\beta_2\in(0,1)$.\\
\textup{(a) (Arbitrary gradient and cell)} For every $\epsilon>0$, $g_{t+1}\in\R$, and $0\leq a<\tilde v_t<b$, let $Q_{\mathrm U}$ and $Q_{\mathrm S}$ denote Update-SR and State-SR between $a$ and $b$. Then
\[
    r_{t+1}(b)
    < \E[r_{t+1}(Q_{\mathrm U}(\tilde v_t))]
    \leq r_{t+1}(\tilde v_t)
    \leq \E[r_{t+1}(Q_{\mathrm S}(\tilde v_t))].
\]
\textup{(b) (Zero-rounding criterion)}
Fix $b > \tilde{v}_t > 0$ independent of $\epsilon$. As $\epsilon\downarrow0$ with $|g_{t+1}| = \mathcal{O}(\epsilon)$, any (possibly randomized) quantizer $Q$ with $Q(\tilde v_t)\in\{0,b\}$ has bounded preconditioner error $\Delta_r$ if and only if \(\Prob(Q(\tilde{v}_t)=0)=\mathcal{O}(\epsilon)\). If \(\Prob(Q(\tilde{v}_t)=0)=\Theta(1)\), then $\Delta_r \to +\infty$ at rate \(\Theta(\epsilon^{-1})\).
\end{proposition}
\end{samepage}

Part~(a) also shows how retaining zero in the second-moment codebook can reduce the preconditioner bias caused by a positive second-moment floor. For $0<\tilde v_t<b$, a zero-excluding quantizer with decoded minimum $b$ stores $b$, giving a next-step preconditioner $r_{t+1}(b)<r_{t+1}(\tilde v_t)$. Compare this with Update-SR between $0$ and the same positive level $b$. Denoting their conditional mean preconditioner errors by $\Delta_r^{\mathrm{NZ}}$ and $\Delta_r^{\mathrm U}$, respectively, part~(a) gives
\[
    \Delta_r^{\mathrm{NZ}}<\Delta_r^{\mathrm U}\leq0.
\]
Thus Update-SR strictly reduces the magnitude of the bias while keeping the conditional mean preconditioner at or below its unquantized value. This comparison holds at a matched decoded floor for every $\epsilon>0$ and every next gradient $g_{t+1}$.

\paragraph{The zero-adjacent cell.}
Let \(\mathcal C_{0b}:=\{0,b\}\) denote the decoded zero-adjacent codebook. Under the assumptions of \Cref{prop:zero-cell-orders}(b), State-SR is unbiased in \(v\) but rounds to zero with probability \(1-\tilde{v}_t/b=\Theta(1)\), so by \Cref{prop:zero-cell-orders}(b), its next-step preconditioner error diverges as \(\Theta(\epsilon^{-1})\); State-RTN behaves the same throughout the lower half of the cell, where it deterministically rounds to zero. By contrast, Update-SR with $a=0$ rounds to zero with probability
\begin{equation}
\label{eq:q-sr-definition}
    \Prob\!\left(
        Q_{h_t,\mathrm{SR},\mathcal C_{0b}}(\tilde v_t)=0
        \,\middle|\,
        \tilde v_t%
    \right)
    =\frac{\epsilon\big(1-\sqrt{\tilde{v}_t/b}\big)}
           {\sqrt{\tilde{v}_t/(1-\beta_2^t)}+\epsilon}
    =\Theta(\epsilon).
\end{equation}
Hence \upsrmethod{} has bounded mean preconditioner error $\Delta_r$. \uprtnmethod{} selects zero only within a \(\Theta(\epsilon^2)\) neighborhood of the origin, so every fixed \(\tilde{v}_t>0\) eventually rounds to \(b\) for sufficiently small \(\epsilon\). We summarize the results in \Cref{tab:zero-cell-orders} with details in \Cref{app:one-step-error-proofs}.

\begin{table}[!h]
\centering
\footnotesize%
\caption{Conditional mean errors for fixed \(0<\tilde{v}_t<b/2\) as \(\epsilon\downarrow0\), with $|g_{t+1}|=\mathcal O(\epsilon)$.}
\label{tab:zero-cell-orders}
\begin{tabular}{@{}cccc@{}}
\toprule
\textbf{Rule} & \(\Prob(Q(\tilde v_t)=0 \mid \tilde{v}_t)\)
& \textbf{State error $\Delta_v$} & \textbf{Precond. error $\Delta_r$} \\
\midrule
State-RTN & \(1\) & \(-\tilde{v}_t\) & \(+\Theta(\epsilon^{-1})\) \\
State-SR & \(1-\tilde{v}_t/b\) & \(0\) & \(+\Theta(\epsilon^{-1})\) \\
Update-RTN & \(0\) & \(b-\tilde{v}_t\) & \(-\Theta(1)\) \\
Update-SR & \(\Theta(\epsilon)\)
& \(b-\tilde{v}_t-\Theta(\epsilon)\) & \(-\Theta(1)\) \\
\bottomrule
\end{tabular}
\end{table}

Thus, preconditioner-space rounding suppresses the probability of storing zero at the cost of upward state bias.
The next proposition examines how this tradeoff can affect optimization on a scalar quadratic. Part~(a) compares complete deterministic trajectories: State-RTN approaches a nonzero two-cycle, whereas Update-RTN converges geometrically to the minimizer. Part~(b) demonstrates that, in the near-optimum regime, State-SR becomes locally \textit{expansive} while Update-SR remains locally \textit{contractive}.  Proofs are provided in \Cref{app:quadratic-rounding-separation}.

\begin{proposition}[Rounding space separation on a scalar quadratic]
\phantomsection\label{prop:quadratic-rounding-space}
Fix $b,\mu>0$. Consider Adam on $\mathcal L(w)=\mu w^2/2$ with $\beta_1=0$, $\beta_2\in(1/2,1)$, and second-moment codebook $\mathcal C_{0b}=\{0,b\}$.\\
\textup{(a) (RTN trajectory)}
Fix $\epsilon>0$ and omit bias correction in Adam. Let $\tau_\epsilon=b\epsilon^2/(\sqrt b+2\epsilon)^2$. For any constant stepsize satisfying ${2(\sqrt{\tau_\epsilon}+\epsilon)}/{\mu} < \eta < {2(\sqrt{b/2}+\epsilon)}/{\mu}$, there is an open interval $I \subset (0,\infty)$ such that, from every common initialization $w_0 \in I$ and $\tilde v_0=0$, State-RTN converges to a stable nonzero two-cycle, whereas \uprtnmethod{} converges geometrically to $w=0$.\\
\textup{(b) (SR one-step shift)}
Fix $t \ge 1$ and consider a family of realized step-$t$ configurations satisfying $\tilde{v}_t \rightarrow 0$, $g_{t+1}^2=\Theta(\tilde v_t)$,
$\epsilon=\smallO(\tilde{v}_t)$, and
$\eta_{t+1}\mu r_{t+1}(0)\to\alpha>2$. Taking expectations only over the randomness of SR at step-$t$ gives
\[
    \E\!\left[
        \log\frac{|w_{t+1}|}{|w_t|}
        \,\middle|\,
        \tilde v_t,g_{t+1}
    \right]
    =\left\{
    \begin{array}{cl}
        \log(\alpha-1)+\smallO(1), & \text{State-SR}\\
    -\Theta\big(\sqrt{\tilde{v}_t}\big), & \text{Update-SR}.
    \end{array}\right.
\]
\end{proposition}

\paragraph{Controlled comparison of rounding choices.}
With FP32 first moments and the same zero-inclusive Dyn4 second-moment codebook, State-RTN and State-SR develop large training-loss gaps, whereas both preconditioner-space rules avoid this deterioration (\Cref{fig:rtn-sr-dynamics}).
Thus, in this controlled run, changing the rounding space avoids the observed instability without removing zero; state-space stochastic rounding alone does not resolve it.

We further compare \uprtnmethod{} and \upsrmethod{} in single-seed GPT-style 834M pretraining with FP32, SDyn4, or NF4 first moments. Their final validation-loss differences, $L_{\mathrm{SR}}-L_{\mathrm{RTN}}$, are $(+1.08,+0.97,-0.10)\times10^{-4}$, respectively, and do not identify a consistent endpoint winner. We select \upsrmethod{} because it tracks the matched 32-bit trajectory more closely during warmup in the FP32-first-moment control.

\begin{figure}[!htbp]
    \centering
    \includegraphics[width=0.6\linewidth]
        {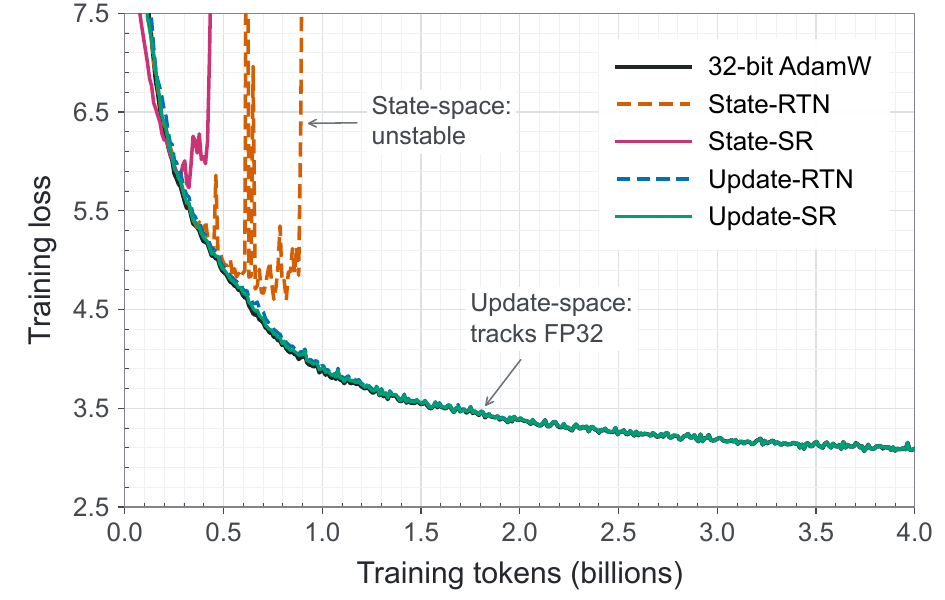}
    \vspace{-7pt}
    \caption{Rounding space controls training stability. Training loss over the first $4$B tokens of GPT-style 834M pretraining, with FP32-$\bm$ and zero-inclusive Dyn4-$\bv$. State-RTN/SR become unstable and exceed the plotted range, while Update-RTN/SR closely track 32-bit AdamW.}
    \label{fig:rtn-sr-dynamics}
\end{figure}

\subsection{ZE-EDEN: Zero-Excluding EDEN calibration}
\label{subsec:zero-excluding}

The second route excludes zero, preventing the zero-endpoint failure by construction but imposing a positive floor that inflates small second moments and suppresses the corresponding large preconditioner entries. For a second-moment block $\bx$, we use the guarded EDEN-inspired calibration~\citep{vargaftik2022eden} to mitigate this distortion:
\[
    \widetilde{\bx}= c_{\tau} \,Q(\bx)
    \quad\text{for}\quad
    \tau > 0,\;
    c_{\tau} = \frac{\lVert \bx\rVert_2^2}{\max\{\langle \bx,Q(\bx)\rangle, \tau s^2\}}.
\]
Here $s>0$ is the guarded absmax base scale, $Q(\bx)$ is the uncalibrated decoded block, and $\tau=10^{-12}$ safeguards the normalized inner product. The implementation evaluates the normalized form and absorbs $c_\tau$ into the block scale. The positive floor can overestimate small second-moment coordinates. When this makes $\langle\bx,Q(\bx)\rangle>\|\bx\|_2^2$, we have $c_\tau<1$, so calibration lowers the decoded block scale and its positive floor. The resulting smaller second-moment values permit larger preconditioner entries, potentially restoring part of the tail suppressed by uncalibrated quantization.

\begin{figure}[!htbp]
\centering
\includegraphics[width=0.95\linewidth]
    {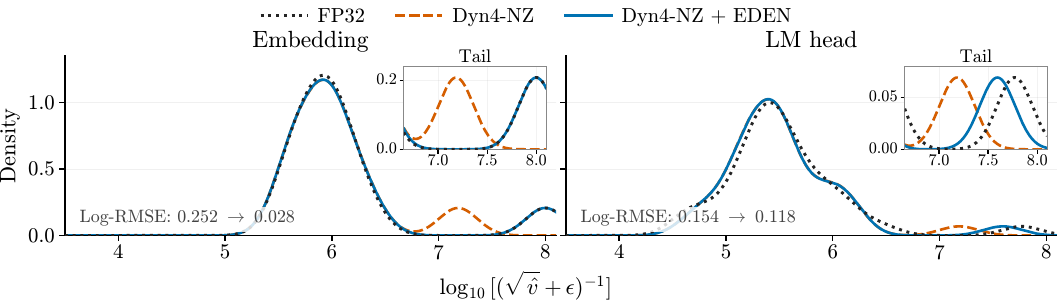}
\vspace{-7pt}
\caption{Next-step preconditioner distributions at 50\% of GPT-style 834M pretraining. EDEN partially restores the large-preconditioner tail suppressed
by Dyn4-NZ in the embedding and LM-head states. Log-RMSE measures
coordinate-wise error from FP32.}
\label{fig:eden-calibration}
\vspace{4pt}
\end{figure}

We denote zero-excluding dynamic 4-bit quantization by Dyn4-NZ. \Cref{fig:eden-calibration} shows that its positive floor suppresses the large-preconditioner tail, whereas EDEN partially restores the FP32 distribution for both the embedding and LM-head states. With FP32 first moments, EDEN lowers the final validation-loss gap of Dyn4-NZ from $0.0049$ to $0.0038$ (three-seed means). With NF4 or SDyn4 first moments, the uncalibrated gap rises to $0.0207$ or $0.0312$ (\Cref{tab:834m-grid}), motivating the first-moment study in \Cref{sec:first-moment-quantization}; \Cref{subsec:attribution} measures EDEN's contribution in that setting. Additional diagnostics in \Cref{app:eden-replays} extend the distributional comparison across training checkpoints and report one-step update errors for Lin4-NZ.

\section{First-Moment Quantization}
\label{sec:first-moment-quantization}

For a fixed preconditioner, the first moment enters the AdamW update linearly, making reconstruction error a useful initial criterion for codebook selection.
However, its storage errors also recur through the exponential moving average, so small instantaneous error need not prevent persistent directional distortion.
We therefore examine both codebook fit and the impact of the rounding rule on long-term training stability.

\subsection{Codebook choice}

\begin{figure}[!htbp]
\vspace{-0.05in}
\centering
\includegraphics[width=0.95\linewidth]
    {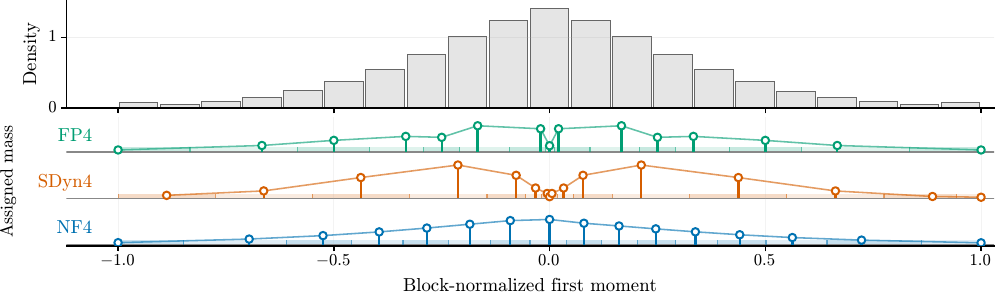}
\vspace{-5pt}
\caption{Empirical distribution of block-normalized FP32 first-moment
values from the layer-12 MLP up-projection at step 15,908 of GPT-style 834M pretraining, together with the probability mass assigned to each
reconstruction level by FP4, SDyn4, and NF4 under RTN.}
\label{fig:codebook-distribution-analysis}
\vspace{4pt}
\end{figure}

TorchAO 4-bit AdamW uses SDyn4 for the first moment. We revisit this
choice using the block-normalized first-moment distribution in
\Cref{fig:codebook-distribution-analysis}. In this snapshot, FP4 and
SDyn4 allocate several levels to narrow cells near zero, each receiving
little probability mass, whereas NF4 spreads its levels across the
high-density region. Under the same block-scaling rule, the mean
squared reconstruction errors are $1.739\times10^{-3}$,
$2.495\times10^{-3}$, and $1.056\times10^{-3}$ for FP4, SDyn4,
and NF4, respectively.

Complete single-seed pretraining of the GPT-style 834M model, with RTN for the first moment throughout, confirms this ordering. The final gaps for SDyn4, FP4, and NF4 are $0.0289$, $0.0272$, and $0.0183$ under ZE-EDEN and $0.0239$, $0.0214$, and $0.0140$ under ZIP-SR. \Cref{subsec:attribution} quantifies the NF4 change with three seeds.

\subsection{Rounding rule for the first moment}
We first use RTN to quantize the first moment $\bm_t$ with the NF4 codebook. This choice remains stable in our smaller GPT-style pretraining experiments, where we observe no comparable late-stage degradation. In GPT-style 2.7B pretraining, however, the NF4-RTN trajectory departs sharply from 32-bit AdamW during the learning-rate cooldown, whereas NF4-SR remains stable (\Cref{fig:gpt27b-m1-divergence}a).

To identify the source of this instability, we conduct matched checkpoint interventions that vary the quantization of the LM-head and non-head first moments independently. These interventions localize the failure to deterministic rounding of the LM-head first moment: applying NF4-RTN only to this state reproduces the instability, while replacing RTN with SR restores the stable trajectory. The complete intervention design and results are reported in \Cref{app:lm-head-localization}. %
Layer-specific sensitivity to optimizer-state quantization has also been reported by \citet{han2025qadammini}, who use stochastic rounding for embedding-layer momentum to address weight-norm explosion.

Having localized the failure, we examine the directional structure of the LM-head quantization error. For the LM-head first moment $\tilde\bm_t(\neq 0)$, define its radial bias as
\[
    b_{\mathrm{rad}}(\tilde\bm_t)={\langle Q(\tilde\bm_t)-\tilde\bm_t, \tilde\bm_t\rangle}/{\|\tilde\bm_t\|_2^2},
\]
which weights the radial bias of each block by its squared norm.
A negative value indicates contraction along the direction of $\tilde\bm_t$, whereas a positive value indicates expansion along that direction. 
Because NF4 spans $[-1,1]$, no value is clipped, and coordinate-wise state-space SR has zero expected radial bias conditional on $\tilde\bm_t$ and its block scales. The trajectory diagnostic therefore examines whether systematic radial contraction occurs in the realized run.

\begin{figure}[!htbp]
\centering
\includegraphics[width=\linewidth]{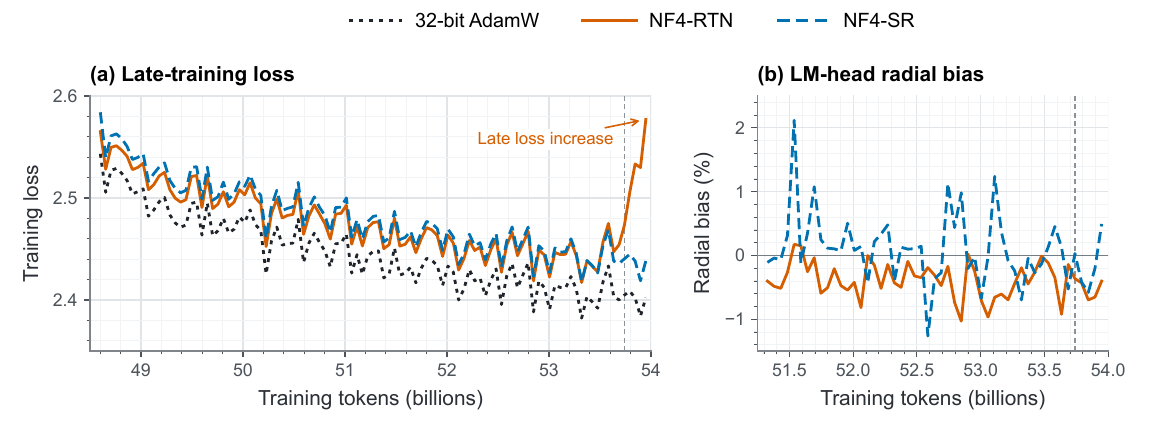}
\vspace{-13pt}
\caption{Late-training instability and LM-head directional bias in GPT-style 2.7B pretraining. (a) Cooldown training loss averaged in 100-update bins; NF4-SR uses SR on the LM-head first moment. (b) Radial bias $b_{\mathrm{rad}}(\tilde\bm_t)$ logged every 100 steps, with FP32 non-head first moments and Dyn4/Update-SR second moments. The vertical dashed line marks the first global gradient-clipping event of the NF4-RTN continuation.}
\label{fig:gpt27b-m1-divergence}
\vspace{6pt}
\end{figure}

NF4-RTN produces a predominantly negative radial bias in the LM-head first moment, whereas the bias under NF4-SR remains centered near zero (\Cref{fig:gpt27b-m1-divergence}b). Since the quantized first moment is repeatedly fed back into its recurrence, the persistent inward bias under RTN can accumulate across iterations, providing a possible explanation for the subsequent loss instability. Additional diagnostics of the matched continuations, including LM-head gradient growth and comparisons with magnitude-based quantization errors, are provided in \Cref{app:lm-head-localization}. Guided by this diagnosis, our final methods use NF4-RTN for the first moment by default and switch only the LM-head first moment to NF4-SR during the final $10\%$ of pretraining.

\section{Experiments}
\label{sec:experiments}

We now evaluate the complete optimizer configurations in \Cref{tab:optimizer-state-configurations}.
The preceding controlled comparisons examine rounding space, EDEN calibration, and first-moment quantization separately; \Cref{subsec:attribution} first combines them into a single attribution at 834M, and the remaining experiments evaluate the complete configurations in pretraining and supervised fine-tuning.

Both methods quantize the first moment with NF4, using RTN by default and switching the LM-head first moment to SR during the final 10\% of training. These shared choices are combined with two second-moment designs: ZIP-SR applies Update-SR with the zero-inclusive Dyn4 codebook, whereas ZE-EDEN applies State-RTN with the zero-excluding Dyn4-NZ codebook and EDEN calibration. Full implementation and training details are given in \Cref{app:exp-details}.

\begin{table}[!htbp]
\centering
\caption{Optimizer-state configurations evaluated in the end-to-end experiments. ZIP-SR and ZE-EDEN share the NF4 first-moment design established in \Cref{sec:first-moment-quantization}, but use the two different second-moment quantization routes developed in \Cref{sec:second-moment-quantization}. The superscript \(\ast\) indicates a phase-dependent rule: the LM-head first-moment uses NF4-RTN before the final 10\% of training and switches to NF4-SR only during that phase.}
\label{tab:optimizer-state-configurations}
\small
\begin{tabular}{@{}llllll@{}}
\toprule
\multirow{2}{*}{\textbf{Configuration}}
& \multicolumn{2}{c}{\textbf{First moment}}
& \multicolumn{3}{c}{\textbf{Second moment}} \\
\cmidrule(lr){2-3}\cmidrule(lr){4-6}
& \textbf{Non-LM-head}
& \textbf{LM-head}
& \textbf{Format}
& \textbf{Rule}
& \textbf{EDEN} \\
\midrule
32-bit
& FP32
& FP32
& FP32
& \textemdash
& $\times$\\
TorchAO 4-bit 
& SDyn4-RTN
& SDyn4-RTN
& Lin4-NZ
& State-RTN
& $\times$\\
\upsrrecipe{} 4-bit
& NF4-RTN
& NF4-RTN/SR$^\ast$
& Dyn4
& \upsrmethod{}
& $\times$\\
\nzedenrecipe{} 4-bit
& NF4-RTN
& NF4-RTN/SR$^\ast$
& Dyn4-NZ
& State-RTN
& $\checkmark$\\
\bottomrule
\end{tabular}
\end{table}

\paragraph{Persistent optimizer-state storage.}
All three 4-bit configurations use packed 4-bit storage for eligible moment tensors and therefore have essentially the same persistent optimizer-state footprint. The codebook and rounding changes require no additional per-parameter state, and ZE-EDEN absorbs its calibration factor into the existing block scale. With block size 128 and one FP32 scale per block for each moment, storing both AdamW moments requires $1.0625$ bytes per parameter on eligible tensors, compared with $8$ bytes per parameter for FP32 moment storage, corresponding to a $7.53\times$ compression. At the scale of our experiments, this corresponds to approximately $19.2$ GB less persistent moment-state storage for the 2.7B-parameter pretraining model and roughly $21$--$56$ GB less for the 3B--8B full-parameter SFT models. To keep the storage comparison controlled, all 4-bit methods use the same tensor-eligibility criterion as the TorchAO baseline. In our model configurations, only normalization-layer moment tensors fall outside this criterion and therefore remain in FP32.

\subsection{Component attribution}
\label{subsec:attribution}

The second-moment format change alone does not explain the final gains
in \Cref{tab:834m-attribution}. Replacing Lin4-NZ with Dyn4-NZ for
$\bv$ increases the validation-loss gap (step 1). Switching $\bm$
from SDyn4 to NF4 gives the largest stepwise reduction (step 2).
Both branches from step 2 then reduce the gap further: EDEN does so
while keeping Dyn4-NZ and State-RTN fixed (step 3a), and the Dyn4
\upsrmethod{} configuration yields a larger reduction (step 3b).
The step 3b comparison changes both the second-moment format and
rounding rule and uses a single run, so it does not isolate the effect
of \upsrmethod{}. Relative to TorchAO, the gaps after steps 3a and 3b
are $28.2\%$ and $43.2\%$ lower, respectively.

The sample SDs for steps 1, 2, and 3a are $8\times10^{-4}$,
$1.0\times10^{-3}$, and $1.5\times10^{-3}$, respectively. Quantized
first moments use RTN throughout, including the LM head.
\Cref{app:component-attribution} reports the full grid across
first-moment settings and both orders of changing the first- and
second-moment choices.

\begin{table}[!htbp]
\centering\footnotesize
\setlength{\tabcolsep}{3.5pt}
\caption{Component attribution on GPT-style 834M. Gaps are relative
to matched 32-bit AdamW. Steps 1, 2, and 3a are means over three random seeds;
steps 0 and 3b are single runs. Steps 3a and 3b both compare with
step 2. Positive Gain denotes a smaller gap; derived values use
unrounded gaps.}
\label{tab:834m-attribution}
\begin{tabular}{@{}ccclccrrr@{}}
\toprule
\multirow{2}{*}{\textbf{Step}}
& \multirow{2}{*}{\textbf{Setting}}
& \textbf{First moment}
& \multicolumn{3}{c}{\textbf{Second moment}}
& \multirow{2}{*}{\textbf{Gap}}
& \multicolumn{2}{c}{\textbf{vs. parent}} \\
\cmidrule(lr){3-3}\cmidrule(lr){4-6}\cmidrule(l){8-9}
& & \textbf{Format} & \textbf{Format} & \textbf{Rule} & \textbf{EDEN}
& & \textbf{$\Delta$} & \textbf{Gain} \\
\midrule
0  & TorchAO 4-bit    & SDyn4 & Lin4-NZ & State-RTN & $\times$
   & $0.0246$ & \textemdash & \textemdash \\
1  & Dyn4-NZ control  & SDyn4 & Dyn4-NZ & State-RTN & $\times$
   & $0.0312$ & $+0.0066$ & $-26.9\%$ \\
2  & NF4 first moment & NF4   & Dyn4-NZ & State-RTN & $\times$
   & $0.0207$ & $-0.0105$ & $33.7\%$ \\
\midrule
3a & \nzedenrecipe{}  & NF4   & Dyn4-NZ & State-RTN & $\checkmark$
   & $0.0177$ & $-0.0031$ & $14.7\%$ \\
3b & \upsrrecipe{}    & NF4   & Dyn4    & \upsrmethod{} & $\times$
   & $0.0140$ & $-0.0067$ & $32.5\%$ \\
\bottomrule
\end{tabular}
\end{table}

\subsection{Pretraining}
\label{subsec:pretraining}

We pretrain GPT-style models at 162M, 405M, 1.4B, and 2.7B parameters and Llama-style models at 130M, 350M, and 1.1B with three paired seeds. \Cref{tab:prelim} reports absolute 32-bit AdamW validation loss and paired 4-bit gaps. Both methods reduce TorchAO's mean gap at every model size. 
A single-seed GPT-style 834M experiment with cosine decay also finds lower gaps for \nzedenrecipe{} and \upsrrecipe{} than for TorchAO (\Cref{app:lr-schedule}).

On GPT-style 2.7B, TorchAO undergoes a sharp late-cooldown divergence. At the validation checkpoint at 95\% of training tokens, before the terminal regression its mean paired gap was \(+0.0754\,(0.0034)\), compared with \(+0.0471\,(0.0018)\) for ZE-EDEN and \(+0.0286\,(0.0024)\) for ZIP-SR, where parentheses report sample SDs. Relative to TorchAO, these correspond to gap reductions of \(37.6\%\) for ZE-EDEN and \(62.1\%\) for ZIP-SR. Hence, both proposed optimizers were already closer to 32-bit AdamW before TorchAO\textquoteright{}s terminal regression. Figure \ref{fig:gpt27b-sr90-seed42-val-loss}, in the appendix, shows the GPT-style 2.7B validation loss trajectories.
The largest reported percentage reduction is therefore the 70.1\% reduction at 1.4B; we do not compute a percentage reduction from the divergent 2.7B endpoint.

\newcommand{\ddagmark}{%
  \smash{\makebox[0pt][r]{%
    \raisebox{1.0ex}{\scriptsize\(\dagger\)}%
    \hspace{-0.3em}}}}
\begin{table}[!htbp]%
\centering

    \caption{Pretraining validation loss: absolute 32-bit loss and mean paired 4-bit gaps to matched 32-bit runs (sample standard deviation); bold marks the smallest gap per row. The final column reports the larger gap reduction achieved by ZE-EDEN or ZIP-SR. All rows use three seeds. The percentage reduction for GPT-style 2.7B is omitted because the \(\dagger\)-marked TorchAO endpoint diverges late in cooldown.}

\label{tab:prelim}

\small
\setlength{\tabcolsep}{2.3pt}
\renewcommand{\arraystretch}{1.0}

\begin{tabular*}{\linewidth}{
    @{\extracolsep{\fill}} ll ccccc @{}
}
\toprule
\multirow{2}{*}{\textbf{Family}} &
\multirow{2}{*}{\textbf{Size}} &
\multirow{2}{*}{\textbf{32-bit loss}} &
\multicolumn{3}{c}{\textbf{4-bit gap to 32-bit AdamW}} &
\multirow{2}{*}{\shortstack{\textbf{Best reduction}\\
                             {vs.\ TorchAO (\%)}}} \\
\cmidrule(lr){4-6}
& & &
\textbf{TorchAO} &
\textbf{\nzedenrecipe{}} &
\textbf{\upsrrecipe{}} &
\\
\midrule

\multirow{4}{*}{GPT-style}
& 162M
& \msd{3.1335}{0.0016}
& \msd{+0.0290}{0.0028}
& \msd{+0.0200}{0.0021}
& \msd{\mathbf{+0.0195}}{0.0020}
& 32.8 \\

& 405M
& \msd{2.8197}{0.0008}
& \msd{+0.0254}{0.0015}
& \msd{\mathbf{+0.0167}}{0.0007}
& \msd{+0.0180}{0.0005}
& 34.3 \\

& 1.4B
& \msd{2.5186}{0.0012}
& \msd{+0.0521}{0.0048}
& \msd{+0.0184}{0.0010}
& \msd{\mathbf{+0.0156}}{0.0026}
& 70.1 \\

& 2.7B
& \msd{2.4044}{0.0019}
& \msd{+1.5435\ddagmark}{0.1437}
& \msd{+0.0535}{0.0024}
& \msd{\mathbf{+0.0343}}{0.0027} 
& \textemdash \\

\midrule

\multirow{3}{*}{Llama-style}
& 130M
& \msd{2.6974}{0.0027}
& \msd{+0.0205}{0.0015}
& \msd{+0.0177}{0.0009}
& \msd{\mathbf{+0.0171}}{0.0003}
& 16.6 \\

& 350M
& \msd{2.4054}{0.0014}
& \msd{+0.0290}{0.0014}
& \msd{\mathbf{+0.0135}}{0.0013}
& \msd{\mathbf{+0.0135}}{0.0007}
& 53.4 \\

& 1.1B
& \msd{2.1467}{0.0002}
& \msd{+0.0486}{0.0029}
& \msd{+0.0226}{0.0026}
& \msd{\mathbf{+0.0198}}{0.0003}
& 59.3 \\

\bottomrule
\end{tabular*}
\end{table}

\subsection{Supervised Fine-tuning}
\label{subsec:sft}

We run full-parameter Supervised Fine-tuning (SFT) on Tulu-3~\citep{lambert2025tulu3} using Qwen3-8B-Base~\citep{yang2025qwen3} and Llama-3.2-3B~\citep{meta2024llama32} over five paired seeds, evaluating MMLU, GSM8K, HumanEval, and IFEval (\Cref{tab:sft-results}). %
Both proposed methods attain lower mean validation loss than TorchAO 4-bit AdamW. Relative to matched 32-bit AdamW, ZE-EDEN's mean paired validation-loss gap is essentially zero on Qwen3-8B and slightly negative on Llama-3.2-3B. Downstream scores remain close to those of 32-bit AdamW, without a consistent advantage across tasks for either
proposed method.
ZE-EDEN attains the lower validation loss on both SFT models, whereas ZIP-SR generally achieves smaller gaps in the pretraining comparisons.

\begin{table}[!htbp]%
\centering
\caption{All SFT runs use AdamW. The 32-bit values are five-seed means, while the 4-bit values are mean paired gaps relative to matched 32-bit runs; parentheses report sample SDs. Downstream scores are percentages.}
\label{tab:sft-results}
\small
\setlength{\tabcolsep}{3pt}
\renewcommand{\arraystretch}{1.15}

\begin{tabular*}{\linewidth}{
    @{\extracolsep{\fill}} l l ccccc @{}
}
\toprule
\textbf{Model} &
\textbf{Configuration} &
\textbf{Tulu val. loss} \(\downarrow\) &
\textbf{MMLU} \(\uparrow\) &
\textbf{GSM8K} \(\uparrow\) &
\textbf{HumanEval} \(\uparrow\) &
\textbf{IFEval strict} \(\uparrow\) \\
\midrule

\multirow{4}{*}{\textbf{\shortstack[l]{Qwen3-\\8B-Base}}} & 32-bit
& \msd{0.55135}{0.00010}
& \msd{78.5}{0.2}
& \msd{82.8}{0.6}
& \msd{72.6}{0.4}
& \msd{49.5}{2.1} \\
& TorchAO 4-bit
& \msd{+0.00060}{0.00009}
& \msd{0.0}{0.1}
& \msd{+0.2}{1.1}
& \msd{-0.1}{1.2}
& \msd{+0.6}{1.8} \\
& \nzedenrecipe{} 4-bit
& \msd{-0.00003}{0.00014}
& \msd{+0.1}{0.1}
& \msd{+0.2}{0.9}
& \msd{+1.1}{1.3}
& \msd{-0.1}{2.4} \\
& \upsrrecipe{} 4-bit
& \msd{+0.00042}{0.00016}
& \msd{+0.1}{0.2}
& \msd{+0.6}{0.6}
& \msd{+0.5}{2.0}
& \msd{-0.1}{1.9} \\

\midrule

\multirow{4}{*}{\textbf{\shortstack[l]{Llama-\\3.2-3B}}} & 32-bit
& \msd{0.73255}{0.00044}
& \msd{52.0}{2.0}
& \msd{34.6}{1.2}
& \msd{33.4}{1.8}
& \msd{37.4}{1.9} \\
& TorchAO 4-bit
& \msd{+0.00175}{0.00012}
& \msd{+1.4}{1.5}
& \msd{+0.4}{1.6}
& \msd{+0.6}{1.4}
& \msd{+0.4}{1.2} \\
& \nzedenrecipe{} 4-bit
& \msd{-0.00035}{0.00016}
& \msd{-0.4}{1.7}
& \msd{-0.2}{1.5}
& \msd{0.0}{1.5}
& \msd{+0.4}{2.4} \\
& \upsrrecipe{} 4-bit
& \msd{+0.00089}{0.00012}
& \msd{+0.4}{1.1}
& \msd{-0.4}{0.8}
& \msd{0.0}{1.4}
& \msd{+0.2}{1.3} \\

\bottomrule
\end{tabular*}
\end{table}

\section{Conclusion}
This work shows that optimizer-state quantization should account for how stored values affect optimization, rather than state-space reconstruction alone. For the second-moment, this perspective motivates \upsrmethod{} for the zero-inclusive route, while EDEN calibration separately improves the zero-excluding route by mitigating distortion in the induced preconditioner distribution. Together with NF4 and SR for the LM-head first-moment, these designs improve over TorchAO across pretraining and fine-tuning settings.

\textbf{Limitations and future work.}
Our theory compares conditional one-step preconditioner biases and optimization dynamics on a toy model, but does not establish general end-to-end convergence for quantized AdamW; extending these results to broader settings remains future work. Targeted LM-head SR alleviates the observed late-training regression, but the cause of the directional first-moment bias and its possible relationship to learning-rate decay remain unresolved. Further directions include extending our methodology below 4-bit optimizer-state quantization and investigating whether it applies to other adaptive optimizers.

\subsection*{AI Use Statement}
Building on the authors' initial ideas to compare the optimizer state and preconditioner errors of different rounding rules, we used ChatGPT (GPT-6 and GPT-5.6 Sol) to help formulate the statements of \Cref{prop:zero-cell-orders,prop:quadratic-rounding-space} for theoretical justification and polish the proof in multiple rounds of discussion.
We used ChatGPT and Claude to give feedback on the experimental design and to polish the writing and organization.
The authors verified all AI-assisted proofs and text and take full responsibility for the manuscript.

\bibliography{4bit_adam_ref}

\clearpage
\appendix

\section{Quantization Background: Codebooks and Rounding Rules}
\label[appendix]{app:quantization-background}

\subsection{4-Bit Formats and Decoded Codebooks}
\label[appendix]{app:quantization-codebooks}

A 4-bit format is specified by its normalized codebook, the set of distinct values that its $16$ codes decode to at unit scale. A normalized codebook may have fewer than $16$ elements; FP4 and Dyn4-NZ below have $15$. 

\begin{definition}[Block quantization and decoded codebook]
\phantomsection\label[definition]{def:decoded-quantization}
Let $\mathcal T\subset[-1,1]$ be the normalized codebook of a 4-bit format.
Partition a flattened tensor into blocks of $B$ consecutive elements.
For a block $\bx\in\R^B$, we use the guarded absmax base scale $s=\max\{\max_{1\leq i\leq B}|\bx_i|,\delta_s\}$, where $\delta_s=10^{-12}$. The decoded codebook is
\[
    \mathcal C=s\,\mathcal T=\{s\,c:c\in\mathcal T\}.
\]
A quantizer writes each coordinate $\bx_i$ as a level $Q(\bx_i)\in\mathcal C$.
Adjacent levels are consecutive elements of $\mathcal C$ in the order of $\R$, not codes with consecutive indices.
\end{definition}

An all-zero block has base scale $\delta_s$. Zero-inclusive codebooks reconstruct it exactly as zero, whereas an uncalibrated zero-excluding codebook reconstructs its positive decoded floor. For ZE-EDEN, $c_\tau=0$ on an all-zero block, so the effective stored scale $c_\tau s$ and the decoded block are zero. We write $\mathcal T_{\mathrm F}\subset[-1,1]$ for the normalized codebook of format $\mathrm F$, and $\mathcal C$ always denotes a decoded codebook.

\paragraph{Dynamic codebook.}
\citet{li2023memory} use the dynamic-exponent (DE) mapping, which
allocates more reconstruction levels near zero. We briefly review its
construction before giving the resulting 4-bit codebooks.
For \(F\geq0\), define uniformly spaced boundaries on \([0.1,1]\) by
\[
    p_{F,j}
    :=
    0.1+\frac{0.9j}{2^F},
    \qquad
    j=0,\ldots,2^F,
\]
and let
\[
    r_{F,k}
    :=
    \frac{p_{F,k}+p_{F,k+1}}{2},
    \qquad
    k=0,\ldots,2^F-1,
\]
denote the corresponding interval midpoints.

For an unsigned \(n\)-bit DE mapping, a nonexceptional code contains
\(E\) leading zeros, one indicator bit, and
\(F=n-E-1\) fraction bits. Its reconstruction value is
\(10^{-E}r_{F,k}\). Including the two exceptional reconstruction values
\(0\) and \(1\), the resulting normalized codebook is
\[
    \mathcal T_{\mathrm{Dyn},n}
    =
    \{0,1\}
    \cup
    \bigcup_{E=0}^{n-2}
    \left\{
        10^{-E}r_{n-E-1,k}
        :
        0\leq k<2^{n-E-1}
    \right\}.
\]
For signed DE, one bit is additionally reserved for the sign, leaving
\(F=n-E-2\) fraction bits. The normalized codebook is
\[
    \mathcal T_{\mathrm{SDyn},n}
    =
    \{0,1\}
    \cup
    \bigcup_{E=0}^{n-2}
    \left\{
        \pm 10^{-E}r_{n-E-2,k}
        :
        0\leq k<2^{n-E-2}
    \right\}.
\]
The exceptional assignment includes \(1\) but not \(-1\); consequently, the signed DE codebook is not exactly sign symmetric, and its most negative element is \(-0.8875\). 
We refer to the unsigned and signed 4-bit instances as Dyn4 and SDyn4, respectively. Substituting \(n=4\) gives $\mathcal T_{\mathrm{Dyn4}}$ and $\mathcal T_{\mathrm{SDyn4}}$.

\paragraph{FP4 (E2M1).}
We use the FP4 codebook generated by \texttt{bitsandbytes} with two exponent bits, one mantissa bit, and four total bits, including the sign bit. The normalized codebook is
\[
    \mathcal T_{\mathrm{FP4}}
    =
    \left\{
        0,\,
        \pm\frac{1}{48},\,
        \pm\frac{1}{6},\,
        \pm\frac{1}{4},\,
        \pm\frac{1}{3},\,
        \pm\frac{1}{2},\,
        \pm\frac{2}{3},\,
        \pm1
    \right\}.
\]
Its smallest positive normalized reconstruction value is $1/48$. The positive- and negative-zero codes decode to the same value, so $\mathcal T_{\mathrm{FP4}}$ has $15$ elements.

\paragraph{Zero-excluding codebook.}
\citet{li2023memory} consider two distinct ways of excluding zero
from the unsigned second-moment codebook. The first, called DE-0 in their
paper, directly removes zero from the DE mapping. We refer to its 4-bit
instance as Dyn4-NZ:
$
    \mathcal T_{\mathrm{Dyn4\text{-}NZ}}
    = \mathcal T_{\mathrm{Dyn4}}\setminus\{0\}
$.
It has $15$ elements with minimum $0.00325$, so a block with scale $s$ has decoded floor $0.00325\,s$.
\citet{li2023memory} subsequently propose a zero-excluding linear mapping that decodes code $i$ to $(i+1)/2^n$ for $i=0,\ldots,2^n-1$. At four bits,
\[
    \mathcal T_{\mathrm{Lin4\text{-}NZ}}
    =
    \left\{
        \frac{1}{16},
        \frac{2}{16},
        \ldots,
        \frac{16}{16}
    \right\},
\]
with decoded floor $s/16$. This linear mapping, rather than DE-0, is the zero-excluding second-moment
mapping proposed and used in their final 4-bit AdamW configuration. In this
work, Dyn4-NZ and Lin4-NZ therefore denote two different positive codebooks.

\paragraph{NormalFloat.}
NF4 is constructed from quantiles of a standard normal distribution, with
an exact zero included and the resulting values normalized to
\([-1,1]\)~\citep{dettmers2023qlora}. Its normalized codebook is approximately
\[
\begin{split}
\mathcal T_{\mathrm{NF4}}
=\{&
-1,\ -0.696193,\ -0.525073,\ -0.394917,\ -0.284441,\ -0.184773,\ -0.091050,\\
&0,\ 0.079580,\ 0.160930,\ 0.246112,\ 0.337915,\ 0.440710,\ 0.562617,\ 0.722957,\ 1
\}.
\end{split}
\]
The displayed values are rounded for readability; our implementation uses
the corresponding floating-point lookup table.

\subsection{Round-to-Nearest and Stochastic Rounding}
\label[appendix]{app:quantization-rounding}

Let $\mathcal C\subset\R$ be a finite set of decoded levels, such as a block's decoded codebook $s\,\mathcal T$ from \Cref{def:decoded-quantization} or the two-level codebook $\mathcal C_{0b}=\{0,b\}$ of \Cref{subsec:zero-inclusive}.
A rounding coordinate is a scalar map
\[
    \phi:
    [\min \mathcal C,\max \mathcal C]
    \to\R
\]
that is either strictly increasing or strictly decreasing. For an input
$x \in \R$ strictly between adjacent values $a<b$ in $\mathcal C$, define the
lower-endpoint interpolation weight
\[
    p_\phi(x;a,b)
    :=
    \frac{|\phi(x)-\phi(b)|}{|\phi(a)-\phi(b)|}
    \in [0,1].
\]

\begin{definition}[Round-to-nearest]
\label[definition]{def:round-to-nearest}
Round-to-nearest writes
\[
    Q_{\phi,\mathrm{RTN},\mathcal C}(x)
    =
    \begin{cases}
        a, & \text{if }p_\phi(x;a,b)>\tfrac12,\\
        b, & \text{if }p_\phi(x;a,b)<\tfrac12.
    \end{cases}
\]
If $p_\phi(x;a,b)=1/2$, RTN resolves the tie using a predetermined
rule.
\end{definition}

\begin{definition}[Stochastic rounding]
\label[definition]{def:stochastic-rounding}
Stochastic rounding writes
\[
    Q_{\phi,\mathrm{SR},\mathcal C}(x)
    =
    \begin{cases}
        a, & \text{with probability }p_\phi(x;a,b),\\
        b, & \text{with probability }1-p_\phi(x;a,b).
    \end{cases}
\]
For one fresh endpoint draw,
\[
    \E\!\left[\phi(Q_{\phi,\mathrm{SR},\mathcal C}(x))\right]
    =
    \phi(x),
\]
where $x$, $\mathcal C$, and $\phi$ are fixed.
\end{definition}

An input in $\mathcal C$ is written as itself, and an input outside $[\min\mathcal C,\max\mathcal C]$ is written as the nearer endpoint. Under absmax scaling, inputs never exceed $\max\mathcal C=s$, so this saturation occurs only at the lower end: a zero-excluding codebook writes inputs below its floor as the floor, and SDyn4 writes $x<-0.8875\,s$ as $-0.8875\,s$. State-RTN and State-SR use the strictly increasing coordinate $\phi(v)=v$. Update-RTN and Update-SR use the strictly decreasing coordinate $\phi(v)=h_t(v)$ and require a nonnegative second-moment codebook.

\newpage
\section{Experimental Details}
\label[appendix]{app:exp-details}

\subsection{Shared settings}
\label[appendix]{app:shared-settings}

All pretraining and supervised fine-tuning (SFT) runs use AdamW
with $(\beta_1,\beta_2)=(0.9,0.95)$, denominator stabilizer
$\epsilon=10^{-8}$, and global gradient-norm clipping at $1.0$.
Both tasks use BF16 mixed-precision computation with FP32
master parameters. Task-specific weight decay, learning-rate
schedules, batch sizes, and training budgets are specified below.

The 32-bit AdamW baseline stores both moments in FP32.
All 4-bit configurations use blocks of $128$ elements, packed
4-bit indices, and one FP32 scale per block. A moment tensor
is eligible for quantization if it contains at least $4096$
elements and its size is divisible by $128$; ineligible moment
tensors remain in FP32. The eligibility criterion is identical
across the 4-bit methods. The parameter update uses the working
moments before their quantized representations are stored for
the next step, as described in \Cref{sec:prelim-related}. At each storage step, block scales are selected by the absmax rule in \Cref{def:decoded-quantization}; ZE-EDEN uses the denominator safeguard $\tau=10^{-12}$ and stores the effective scale $c_\tau s$ by absorbing its calibration factor into the block scale.

The main pretraining and SFT comparisons use the configurations
in \Cref{tab:optimizer-state-configurations}. For ZIP-SR and
ZE-EDEN, the first moment uses NF4-RTN by default, with only the
LM-head first moment switching to NF4-SR after $90\%$ of training.
In the GPT-style 834M component comparisons, quantized first moments use RTN throughout training. The LM-head rule is therefore inactive in \Cref{tab:834m-attribution,tab:834m-grid}.
We report absolute mean metrics for 32-bit AdamW and mean paired
differences for the 4-bit methods. Each paired difference
subtracts the matched 32-bit result for the same seed;
parentheses denote sample standard deviations.
Single-seed component results are identified separately.

\subsection{Pretraining}
\label[appendix]{app:pretraining-details}

\paragraph{Models and data.}
The architectures are listed in \Cref{tab:arch}. The main comparison in \Cref{tab:prelim} uses GPT-style models at 162M, 405M, 1.4B, and 2.7B parameters and Llama-style models at 130M, 350M, and 1.1B parameters. The GPT-style 834M model is used for the component studies. All models are pretrained on FineWeb-Edu~\citep{penedo2024fineweb}. GPT-style models use the GPT-2 tokenizer~\citep{radford2019gpt2} and context length $2048$; Llama-style models use the Llama~2 tokenizer~\citep{touvron2023llama2} and context length $1024$.

\begin{table}[!htbp]
\centering
\small
\caption{Architectures used for pretraining and component studies.
Size denotes the rounded parameter count; $L$, $d_{\mathrm{model}}$,
heads, $n_{kv}$, and FFN denote the layer count, hidden width,
attention heads, key/value heads, and feed-forward width.
GPT-style models use RoPE/GELU with untied embeddings.
Llama-style models use RoPE/SwiGLU/RMSNorm, including an output
norm, with $n_{kv}=4$.}
\label{tab:arch}
\begin{tabular}{@{}llrrrrr@{}}
\toprule
\textbf{Family} & \textbf{Size} & $\boldsymbol{L}$
& $\boldsymbol{d_{\mathrm{model}}}$ & \textbf{Heads}
& $\boldsymbol{n_{kv}}$ & \textbf{FFN} \\
\midrule
\multirow{5}{*}{GPT-style}
& 162M & 12 & 768  & 12 & --- & 3072 \\
& 405M & 24 & 1024 & 16 & --- & 4096 \\
& 834M & 24 & 1536 & 16 & --- & 6144 \\
& 1.4B & 24 & 2048 & 16 & --- & 8192 \\
& 2.7B & 32 & 2560 & 32 & --- & 10240 \\
\midrule
\multirow{3}{*}{Llama-style}
& 130M & 12 & 768  & 12 & 4 & 2304 \\
& 350M & 22 & 1024 & 16 & 4 & 3392 \\
& 1.1B & 22 & 2048 & 32 & 4 & 5632 \\
\bottomrule
\end{tabular}
\end{table}

\paragraph{Training details.}
Each model is trained with a budget of $20$ tokens per parameter~\citep{hoffmann2022training}. The effective batch size is $256$ sequences across eight GPUs. GPT-style 405M, 1.4B, and 2.7B and Llama-style 1.1B use a per-GPU micro-batch of $16$ with two gradient-accumulation steps. The remaining configurations use a per-GPU micro-batch of $32$ without gradient accumulation.

Pretraining uses weight decay $\lambda=0.1$ and a warmup-stable-decay (WSD) schedule. The learning rate increases linearly to its peak during the first $10\%$ of updates, remains constant for the next $80\%$, and decays to zero during the final $10\%$. During cooldown, the multiplier is $1-\sqrt{u}$, where $u\in[0,1]$ is the normalized progress through cooldown.
\Cref{tab:fp32-lr-selection} reports the FP32 AdamW learning-rate sweeps used to select the model-specific peak rates.
For GPT-style 2.7B, we fix the peak rate at $1.5\times10^{-4}$ across methods and seeds without a sweep. For each model size, the selected learning rate and the remaining training hyperparameters are shared across optimizer configurations.

\begin{table}[H]
\centering
\caption{Peak learning rates for pretraining. Each sweep uses $20$ training tokens per parameter. The final column reports the minimum observed validation loss at the precision retained in the experiment logs. The GPT-style 2.7B rate is fixed without a sweep.}
\label{tab:fp32-lr-selection}
\begin{tabular}{@{}lllcc@{}}
\toprule
\textbf{Family} & \textbf{Size} & \textbf{Candidate peak LRs}
& \textbf{Selected LR} & \textbf{Val. loss} \\
\midrule
\multirow{5}{*}{GPT-style}
 & 162M  & $\{2,1,0.6,0.3\}\!\times\!10^{-3}$       & $1\!\times\!10^{-3}$ & 3.135 \\
 & 405M & $\{1,0.6,0.3,0.15\}\!\times\!10^{-3}$  & $6\!\times\!10^{-4}$ & 2.818 \\
 & 834M  & $\{0.6,0.3,0.25,0.15\}\!\times\!10^{-3}$ & $6\!\times\!10^{-4}$ & 2.641 \\
 & 1.4B     & $\{0.3,0.2,0.15,0.1\}\!\times\!10^{-3}$  & $2\!\times\!10^{-4}$ & 2.517 \\
 & 2.7B & Not swept & $1.5\!\times\!10^{-4}$ & --- \\
\midrule
\multirow{3}{*}{Llama-style}
 & 130M & $\{4,3,2,1,0.6,0.3\}\!\times\!10^{-3}$       & $2\!\times\!10^{-3}$ & 2.698 \\
 & 350M & $\{2, 1.5, 1,0.6,0.3,0.15\}\!\times\!10^{-3}$  & $1\!\times\!10^{-3}$ & 2.406 \\
 & 1.1B & $\{0.6,0.3,0.2,0.1\}\!\times\!10^{-3}$ & $3\!\times\!10^{-4}$ & 2.148 \\
\bottomrule
\end{tabular}
\end{table}

Each run uses one node with eight H200 GPUs and activation checkpointing and PyTorch DistributedDataParallel (DDP)~\citep{li2020pytorchdistributed}, except that GPT-style 2.7B pretraining runs use one node with eight B200 GPUs.

\paragraph{Evaluation.}
Each pretokenized cache has a fixed validation split of $50{,}000$ sequences. The main pretraining comparisons use three paired seeds and the reporting convention in \Cref{app:shared-settings}.

\subsection{Supervised fine-tuning}
\label[appendix]{app:sft-details}

\paragraph{Models and data.}
We perform full-parameter SFT of
Qwen3-8B-Base~\citep{yang2025qwen3} and
Llama-3.2-3B~\citep{meta2024llama32} on the Tulu-3 SFT dataset~\citep{lambert2025tulu3}. We first reserve $2\%$ of the Tulu-3 mixture for validation using a deterministic split with seed $42$. From the remaining pool, we select a fixed subset of $50{,}000$ training examples, shared across methods and training seeds.
Examples are tokenized with the corresponding model tokenizer,
truncated to $2048$ tokens, and processed without sequence
packing. The loss is computed only on the final assistant
response; preceding conversation and padding tokens are masked.

\paragraph{Training details.}
Each run trains for one epoch with an effective batch size of
$64$ examples. A per-device batch size of one across eight GPUs,
with eight gradient-accumulation steps, gives $782$ optimizer
updates, including the final partial batch.

Both models use zero weight decay and a peak learning rate of
$10^{-5}$. This rate is selected by validation loss in a
seed-$42$ FP32 AdamW sweep and then shared across optimizer
variants. The schedule uses linear warmup during the first
$3\%$ of updates, followed by cosine decay to zero.

Each run uses eight H200 GPUs and activation checkpointing. Qwen3-8B uses DeepSpeed ZeRO-2~\citep{rajbhandari2020zero} for every optimizer, whereas Llama-3.2-3B uses PyTorch DistributedDataParallel (DDP)~\citep{li2020pytorchdistributed}.

\paragraph{Evaluation.}
We evaluate the final model from each run over five paired training seeds, $42$--$46$. We report Tulu validation loss, 5-shot MMLU~\citep{hendrycks2021mmlu}, GSM8K~\citep{cobbe2021gsm8k}, HumanEval~\citep{chen2021humaneval}, and prompt-level strict IFEval~\citep{zhou2023ifeval}. MMLU is macro-averaged over its $57$ subjects. For IFEval, Qwen3-8B uses its corresponding chat template with thinking disabled, whereas Llama-3.2-3B uses the plain template used in training. MMLU, GSM8K, and HumanEval use the standard task prompts without an applied chat template. Results follow the paired reporting convention in \Cref{app:shared-settings}.

\newpage
\section{Additional Diagnostics}
\label[appendix]{app:moment-diag}

We first report the GPT-style 834M component grid underlying \Cref{tab:834m-attribution}, followed by a cosine learning-rate schedule ablation study. Next, we examine EDEN calibration with Lin4-NZ and compare one-step update errors for Lin4-NZ and Dyn4-NZ across training checkpoints. Finally, we examine GPT-style 2.7B training dynamics through additional rounding diagnostics, and validation trajectories for the complete optimizer configurations.

\subsection{Component attribution on GPT-style 834M pretraining}
\label[appendix]{app:component-attribution}

\Cref{tab:834m-grid} expands the Dyn4-NZ and Dyn4/Update-SR branches of \Cref{tab:834m-attribution} across first-moment settings, including the EDEN ablation with FP32 first moments. The TorchAO Lin4-NZ reference appears as step 0 of the main table; Lin4-NZ was not crossed with the other first-moment settings in this grid. The FP4 runs are reported in \Cref{sec:first-moment-quantization}, the state-space rules, which become unstable with zero-inclusive Dyn4, in \Cref{fig:rtn-sr-dynamics}, and the \uprtnmethod{} comparison in \Cref{subsec:zero-inclusive}.

\paragraph{Order sensitivity.}
The grid lets us change the first-moment codebook before or after modifying the second moment. For EDEN, the order changes the observed reduction in the mean gap: adding EDEN under SDyn4 lowers it by $0.39\times10^{-3}$, whereas adding it after switching to NF4 lowers it by $3.05\times10^{-3}$. Replacing Dyn4-NZ State-RTN with Dyn4 Update-SR lowers the displayed gap by about $7\times10^{-3}$ under either first-moment codebook. The Dyn4/Update-SR endpoints are single runs compared with three-seed Dyn4-NZ means. This replacement also changes the second-moment codebook, the rounding coordinate, and RTN to SR together.

\begin{table}[!htbp]
\centering\footnotesize
\caption{GPT-style 834M component grid. Validation-loss gaps to 32-bit AdamW. Rows vary the first moment (RTN throughout); columns vary the second moment. Entries with parentheses report mean (sample standard deviation) over three seeds; the others are single runs.}
\label{tab:834m-grid}
\begin{tabular}{@{}lrrc@{}}
\toprule
& \multicolumn{2}{c}{\textbf{Dyn4-NZ (State-RTN)}}
& \multicolumn{1}{c}{\textbf{Dyn4}} \\
\cmidrule(lr){2-3}\cmidrule(l){4-4}
\textbf{$\bm$}
& \multicolumn{1}{c}{\textbf{Raw}}
& \multicolumn{1}{c}{\textbf{+ EDEN}}
& \multicolumn{1}{c}{\textbf{\upsrmethod{}}} \\
\midrule
FP32  & \msd{0.0049}{0.0001}  & \msd{0.0038}{0.0003}
      & \textemdash \\
SDyn4 & \msd{0.0312}{0.0008} & \msd{0.0308}{0.0027}
      & \mone{0.0239} \\
NF4   & \msd{0.0207}{0.0010} & \msd{0.0177}{0.0015}
      & \mone{0.0140} \\
\bottomrule
\end{tabular}
\end{table}

\subsection{Sensitivity to the learning-rate schedule}
\label[appendix]{app:lr-schedule}

The main pretraining experiments use warmup-stable-decay (WSD).
We also evaluate the complete optimizer configurations on GPT-style
834M with cosine learning-rate decay. The ZE-EDEN and ZIP-SR runs
include the LM-head first-moment switch to NF4-SR during the final
$10\%$ of training.

\begin{table}[!htbp]
\centering\footnotesize
\caption{GPT-style 834M pretraining with cosine learning-rate decay.
Final validation loss at update 31,810 for seed 42. Gaps to the 32-bit AdamW run are shown directly.}
\label{tab:834m-cosine}
\begin{tabular}{@{}lrr@{}}
\toprule
\textbf{Configuration}
& \textbf{Validation loss}
& \textbf{Gap to 32-bit} \\
\midrule
32-bit AdamW      & $2.6274$ & \textemdash \\
TorchAO 4-bit     & $2.6554$ & $0.0280$ \\
\nzedenrecipe{}   & $2.6421$ & $0.0147$ \\
\upsrrecipe{}     & $2.6385$ & $0.0111$ \\
\bottomrule
\end{tabular}
\end{table}

With cosine decay, both recipes have a smaller gap to 32-bit
AdamW than TorchAO in this seed, as they do in the WSD pretraining
experiments. This run tests whether the ordering persists under
cosine decay; the WSD component runs at 834M use RTN for the
first moment throughout and therefore do not provide matched
schedule comparisons for the complete recipes.

\subsection{EDEN calibration for zero-excluding second-moment quantization}
\label[appendix]{app:eden-replays}

The main text examines EDEN calibration with Dyn4-NZ for the second
moment. Here we extend the distributional analysis to Lin4-NZ, the
zero-excluding linear codebook described in
\Cref{app:quantization-codebooks}, and compare one-step update errors
for both codebooks across training.

Using the GPT-style 834M model pretraining, we replay one update from 32-bit
AdamW checkpoints saved at 25\%, 50\%, and 75\% of training.
Within each codebook, the raw and EDEN-calibrated branches share
stored indices, the recorded post-clipping gradient, and FP32
fallback coordinates; only the effective decoded block scales differ.

\paragraph{Preconditioner distributions.}
\Cref{fig:app-eden} compares next-step preconditioner distributions
for FP32, Lin4-NZ, and Lin4-NZ with EDEN calibration, using embedding
and LM-head states. At all three sampled checkpoints, uncalibrated
Lin4-NZ produces floor-induced peaks and suppresses the
large-preconditioner tail. EDEN calibration partially restores this
tail, although residual distributional distortion remains.

\begin{figure}[!htbp]%
\centering
\includegraphics[width=\linewidth]{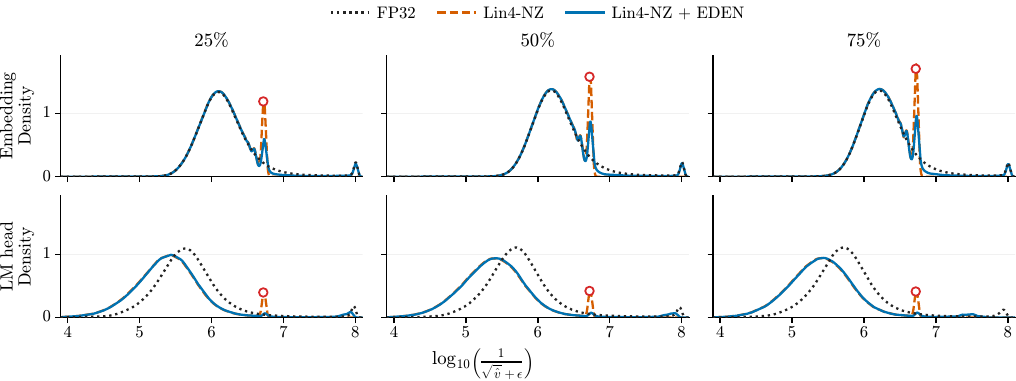}
\vspace{-18pt}
\caption{Matched next-step preconditioner distributions at 25\%, 50\%, and 75\% of training. Columns show checkpoints and rows show embedding and LM-head states. Curves compare FP32, Lin4-NZ, and Lin4-NZ with EDEN; open circles mark the Lin4-NZ floor atom.}
\label{fig:app-eden}
\end{figure}

\paragraph{One-step update errors.}
At the first resumed update, we measure the relative full-vector
adaptive-update error ${\lVert u_{\mathrm{4bit}}-u_{32}\rVert_2}/{\lVert u_{32}\rVert_2}$,
where \(u\) excludes the common learning-rate multiplier and weight decay.
At the 25\%, 50\%, and 75\% checkpoints, EDEN reduces the Lin4-NZ
relative error from \(30.99\%\) to \(30.58\%\), from \(28.80\%\)
to \(28.34\%\), and from \(27.87\%\) to \(27.38\%\), respectively.
For Dyn4-NZ, the corresponding changes are \(15.45\%\) to \(15.41\%\),
\(12.83\%\) to \(12.77\%\), and \(11.83\%\) to \(11.79\%\).
Thus, EDEN yields small but consistent reductions at these three
checkpoints, while Dyn4-NZ has lower error than Lin4-NZ under this
one-step replay metric.

\subsection{GPT-style 2.7B training dynamics and LM-head diagnostics}
\label[appendix]{app:lm-head-localization}

\paragraph{Checkpoint interventions.}
Following the comparison in \Cref{fig:gpt27b-m1-divergence}, we examine matched checkpoint interventions to isolate the effect of LM-head first-moment rounding. To distinguish insufficient four-bit capacity from the deterministic rounding path itself, we branch matched continuations from the same accumulated step-97,831 checkpoint, before the sharp loss increase. The continuations preserve the model weights, numerical optimizer state, data order, and learning-rate schedule; only the stated future first-moment storage rule changes. Every branch keeps the second-moment fixed to Dyn4 Update-SR with an FP32 current read, thereby isolating the LM-head first-moment intervention.

\begin{table}[!htbp]%
\centering
\caption{{Localization of the late-training instability in GPT-style 2.7B pretraining.} Matched continuations vary the quantization of the LM-head and non-head first moments while holding the remaining training configuration fixed.}
\label{tab:app-gpt27b-lmhead-localization}
\begin{tabular}{l l c c}
\toprule
LM-head $\bm_t$ & Non-head $\bm_t$ & Final train loss & Final val.\ loss \\
\midrule
FP32    & NF4-RTN & 2.392 & 2.440 \\
NF4-RTN & FP32    & 2.529 & 2.579 \\
NF4-SR  & FP32    & 2.391 & 2.440 \\
\bottomrule
\end{tabular}
\end{table}

Keeping the LM-head first moment in FP32 prevents the loss increase even when the remaining first moments use NF4-RTN. Conversely, applying NF4-RTN only to the LM-head first moment reproduces the instability. Replacing RTN with SR for this state restores the original training trajectory, showing that the failure is associated with deterministic LM-head rounding rather than 4-bit LM-head quantization itself.

\paragraph{Gradient exploding.}
The main text identifies a persistent inward radial bias in the LM-head first moment under NF4-RTN. We next examine the corresponding gradient dynamics of the matched continuations. Their LM-head gradient norms remain comparable initially but separate rapidly near the end of cooldown (\Cref{fig:app-gpt27b-gradient-growth}).
The LM-head pre-clipping gradient norm under NF4-RTN reaches $2.044$, whereas NF4-SR remains below $0.388$. This separation connects the LM-head quantization behavior examined in the main text with the gradient growth that accompanies the subsequent loss increase.

\vspace{-0.08in}
\begin{figure}[H]
\centering
\includegraphics[width=0.6\linewidth]{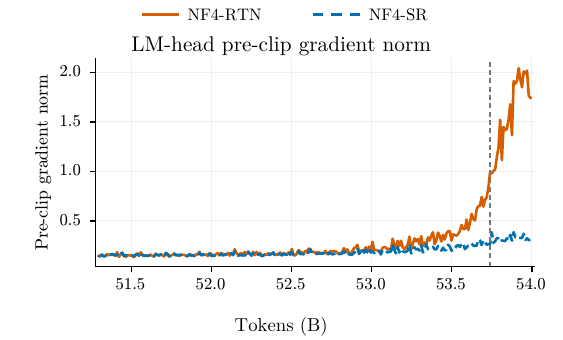}
\vspace{-0.1in}
\caption{{LM-head gradient growth in the matched continuations.} The solid orange curve denotes NF4-RTN and the dashed blue curve denotes NF4-SR. The vertical dashed line marks step $102{,}500$ ($53.74$B tokens), at which the NF4-RTN continuation first triggers global gradient clipping.}
\label{fig:app-gpt27b-gradient-growth}
\end{figure}

\paragraph{Magnitude-based diagnostics.}
Let $\tilde\bm_t$ be the working first moment before storage.
The relative RMS quantization error is $\|Q(\tilde\bm_t)-\tilde\bm_t\|_2/\|\tilde\bm_t\|_2$, and
the nonzero-to-zero fraction is the fraction of nonzero coordinates of
$\tilde\bm_t$ that are written as zero. Both measure the error of a single write.
The shadow-state error measures accumulated distortion in the working first moment before storage. A shadow moment $\bm^{\mathrm{sh}}_t$, initialized from the decoded step-97,831 checkpoint state, follows the unquantized recurrence $\bm^{\mathrm{sh}}_t=\beta_1\bm^{\mathrm{sh}}_{t-1}+(1-\beta_1)\bg_t$ at every subsequent optimizer step. The reported error is $\|\tilde\bm_t-\bm^{\mathrm{sh}}_t\|_2/\|\bm^{\mathrm{sh}}_t\|_2$. This diagnostic includes the effects of earlier stored-state errors but excludes the rounding error of the current write.

The stability of NF4-SR is not explained by smaller quantization errors. \Cref{tab:app-gpt27b-error-metrics} shows that NF4-SR has larger per-write and accumulated errors, despite avoiding both the persistent inward bias and the gradient exploding.

\begin{table}[H]
\centering
\caption{{Magnitude-based diagnostics for the matched LM-head continuations.} Values are means over $91$ recorded evaluations; lower is better. The first two rows measure instantaneous storage-rounding error; the last measures accumulated deviation from the shadow recurrence in the working first moment before storage.}
\label{tab:app-gpt27b-error-metrics}
\begin{tabular}{l c c}
\toprule
Diagnostic & NF4-RTN & NF4-SR \\
\midrule
Relative RMS quantization error (\%) & 9.21 & 14.75 \\
Nonzero-to-zero fraction (\%) & 9.16 & 12.08 \\
Accumulated shadow-state error (\%) & 20.95 & 32.18 \\
\bottomrule
\end{tabular}
\end{table}

NF4-RTN performs better under all three magnitude-based diagnostics but nevertheless develops the late-training instability. These aggregate error measures therefore do not distinguish the stable continuation from the unstable one. The persistent radial bias and subsequent gradient growth provide complementary diagnostics of the observed failure.

\paragraph{End-to-end validation dynamics.}
\Cref{fig:gpt27b-sr90-seed42-val-loss} shows the validation loss curve of the complete optimizer configurations in \Cref{tab:optimizer-state-configurations}.
The methods remain closely aligned before cooldown. During cooldown, ZE-EDEN and ZIP-SR remain close to 32-bit AdamW, while TorchAO undergoes a sharp terminal regression.

\begin{figure}[!htbp]
\centering
\includegraphics[width=0.95\linewidth]
    {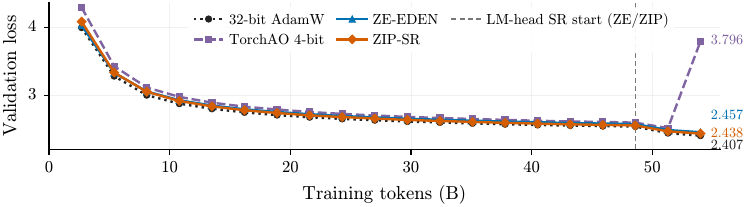}
    \vspace{-0.1in}
\caption{Validation loss during GPT-style 2.7B pretraining for a single run per configuration (seed 42). The dashed line marks the switch from NF4-RTN
to NF4-SR for the LM-head first moment at $90\%$ of training
in ZE-EDEN and ZIP-SR. TorchAO exhibits a sharp regression, whereas ZE-EDEN and ZIP-SR remain close to
32-bit AdamW.}
\label{fig:gpt27b-sr90-seed42-val-loss}
\end{figure}

\newpage
\section{Omitted Proofs in Section~\ref{subsec:zero-inclusive}}
\label[appendix]{app:rounding-proofs}

\subsection{One-Step State and Preconditioner Errors}
\label[appendix]{app:one-step-error-proofs}

\begin{proof}[Proof of \Cref{prop:zero-cell-orders}.]
\textbf{(a)} Fix the quantities in part~(a), and write $v=\tilde v_t$ and $r=r_{t+1}$.
The map $r$ is convex on $[0,\infty)$: the scalar map
$u\mapsto(\sqrt u+\epsilon)^{-1}$ is convex there, and its argument
in $r$ is affine. Since $\E[Q_{\mathrm S}(v)]=v$, Jensen's inequality
gives $\E[r(Q_{\mathrm S}(v))]\geq r(v)$.

For Update-SR, we prove concavity of $F:=r\circ h_t^{-1}$
by establishing convexity of its increasing inverse. Define
\[
    \kappa:=\sqrt{\frac{\beta_2(1-\beta_2^t)}
                           {1-\beta_2^{t+1}}}\in(0,1),
    \qquad
    c:=\frac{(1-\beta_2)g_{t+1}^2}{1-\beta_2^{t+1}}\geq0.
\]
Since $r$ and $h_t$ are strictly decreasing, $F$ is strictly
increasing. For $0<z\leq(\epsilon+\sqrt c)^{-1}$, its inverse is
\[
    F^{-1}(z)=\frac{\kappa z}{D(z)},
    \qquad
    D(z):=\sqrt{(1-\epsilon z)^2-cz^2}+\kappa\epsilon z.
\]
The square-root term is concave by direct computation, so $D$ is concave.
Moreover, $D(z)>0$ and
$D(z)\leq1-\epsilon(1-\kappa)z\leq D(0)=1$.
By concavity, this bound also implies that $D$ is nonincreasing.
The ratio $\kappa z/D(z)$ is therefore convex by the univariate convex-ratio rule \citep[Exercise~3.32(c)]{boyd2004convex}.
Thus $F^{-1}$ is convex and increasing, and $F$ is concave.

Using $\E[h_t(Q_{\mathrm U}(v))]=h_t(v)$, Jensen's inequality gives
\[
    \E[r(Q_{\mathrm U}(v))]
    =\E[F(h_t(Q_{\mathrm U}(v)))]
    \leq F\!\left(\E[h_t(Q_{\mathrm U}(v))]\right)
    =r(v).
\]
Finally, $r$ is strictly decreasing and $Q_{\mathrm U}(v)$ assigns
positive probability to $a<b$, so $\E[r(Q_{\mathrm U}(v))]>r(b)$.
Together these inequalities prove part~(a).

\textup{(b)} Under the gradient assumption, $r(0)=\Theta(\epsilon^{-1})$, whereas $r(v)$ and $r(b)$ remain bounded. For $p_0=\Prob(Q(v)=0)\in[0,1]$,
\[
    \Delta_r=p_0r(0)+(1-p_0)r(b)-r(v)
            =p_0r(0)+\mathcal O(1),
\]
with a remainder uniform in $p_0$. Since $\epsilon r(0)$ is bounded above and away from zero, $\Delta_r$ is bounded if and only if $p_0=\mathcal O(\epsilon)$. If $p_0=\Theta(1)$, the same identity gives $\Delta_r=+\Theta(\epsilon^{-1})$.
\end{proof}

\paragraph{Comparison with a matched positive floor.}
For $0<v=\tilde v_t<b$, a zero-excluding quantizer with decoded minimum $b$ stores $b$ deterministically. Write $r=r_{t+1}$ and let $Q_{\mathrm U}$ denote Update-SR between $0$ and $b$.
The corresponding mean errors
$\Delta_r^{\mathrm{NZ}}=r(b)-r(v)$ and
$\Delta_r^{\mathrm U}=\E[r(Q_{\mathrm U}(v))]-r(v)$ satisfy
$\Delta_r^{\mathrm{NZ}}<\Delta_r^{\mathrm U}\leq0$ by part~(a).
With $p_0=\Prob(Q_{\mathrm U}(v)=0)\in(0,1)$, the reduction in absolute mean error is
\[
    |\Delta_r^{\mathrm{NZ}}|-|\Delta_r^{\mathrm U}|
    =p_0\bigl[r(0)-r(b)\bigr]>0.
\]
This comparison holds for every $\epsilon>0$ and every next gradient at the same decoded floor $b$.

\paragraph{Derivation of \Cref{tab:zero-cell-orders}.}
Fix $0<v=\tilde v_t<b/2$ and assume $|g_{t+1}|=\mathcal O(\epsilon)$.
Write $r=r_{t+1}$ and define
\[
    A=\sqrt{1-\beta_2^t}\left(\frac{1}{\sqrt v}-\frac{1}{\sqrt b}\right)>0,
    \qquad
    \kappa=\sqrt{\frac{\beta_2(1-\beta_2^t)}{1-\beta_2^{t+1}}}\in(0,1).
\]
Let $\gamma_\epsilon=(|g_{t+1}|/\epsilon)
\sqrt{(1-\beta_2)/(1-\beta_2^{t+1})}$, which is bounded by the
gradient assumption. Evaluating $r$ at zero and expanding at the
fixed positive inputs $v,b$ yields
\[
    r(0)=\frac{1}{\epsilon(1+\gamma_\epsilon)},
    \qquad r(v)-r(b)=\frac{A}{\kappa}+\mathcal O(\epsilon).
\]
For a zero-write probability $p_0$, the errors are
$\Delta_v=b-v-bp_0$ and
$\Delta_r=r(b)-r(v)+p_0[r(0)-r(b)]$.
State-RTN has $p_0=1$ and hence $\Delta_v=-v$.
State-SR has $p_0=1-v/b$ and hence $\Delta_v=0$.
Part~(b) gives $\Delta_r=+\Theta(\epsilon^{-1})$ for both rules.

The Update-RTN boundary solves $h_t(v)=(h_t(0)+h_t(b))/2$, giving
\begin{equation}
\label{eq:preconditioner-rtn-boundary}
    v_{\mathrm{th}}
    =\frac{b\epsilon^2}
           {\bigl(\sqrt{b/(1-\beta_2^t)}+2\epsilon\bigr)^2}
    =(1-\beta_2^t)\epsilon^2+\mathcal O(\epsilon^3).
\end{equation}
For sufficiently small $\epsilon$, the fixed input $v>0$ exceeds this
boundary, so $p_0=0$, $\Delta_v=b-v$, and
$\Delta_r=-A/\kappa+\mathcal O(\epsilon)=-\Theta(1)$.

For Update-SR, \eqref{eq:q-sr-definition} gives $p_0=A\epsilon+\mathcal O(\epsilon^2)$. Substituting into the error identities yields
\[
\begin{aligned}
    \Delta_v&=b-v-bA\epsilon+\mathcal O(\epsilon^2),\\
    \Delta_r&=-A\left(\frac{1}{\kappa}-\frac{1}{1+\gamma_\epsilon}\right)
              +\mathcal O(\epsilon).
\end{aligned}
\]
Since $1/\kappa-1/(1+\gamma_\epsilon)\geq1/\kappa-1>0$ and $\gamma_\epsilon$ is bounded, $\Delta_r=-\Theta(1)$. This establishes the remaining row of the table.

\subsection{Rounding-Space Separation on a Scalar Quadratic}
\label[appendix]{app:quadratic-rounding-separation}

\begin{proof}[Proof of Proposition~\ref{prop:quadratic-rounding-space}\textup{(a)}.]
Fix a stepsize $\eta$ and set $c=\sqrt{1-\beta_2}$. For either RTN rule, let $Q$ denote its quantizer and $(\tilde v_k,w_k)$ its working moments and iterates. In this part, Update-RTN uses the coordinate $h(v)=(\sqrt v+\epsilon)^{-1}$.
Starting from $\tilde v_0=0$, the recurrences are
\[
    \tilde v_k=\beta_2Q(\tilde v_{k-1})+(c\mu w_{k-1})^2,
    \qquad
    w_k=w_{k-1}\left(1-\frac{\eta\mu}{\sqrt{\tilde v_k}+\epsilon}\right).
\]
Define the candidate cycle amplitude by
$w_\eta=(\eta\mu/2-\epsilon)/(c\mu)$.
The stepsize bounds give
$\sqrt{\tau_\epsilon}<c\mu w_\eta<\sqrt{b/2}$.
Choose $\delta>0$ sufficiently small that the closed interval
$J=[w_\eta-\delta,w_\eta+\delta]\subset(0,\infty)$ satisfies
\[
    \tau_\epsilon<(c\mu y)^2<b/2,
    \qquad c\mu y+\epsilon<\eta\mu
    \quad\text{for all }y\in J,
\]
and take $I=\operatorname{int}J$. Fix a common initialization $w_0\in I$.

For State-RTN, suppose $Q(\tilde v_{k-1})=0$ and
$y=|w_{k-1}|\in J$. Then $\tilde v_k=(c\mu y)^2<b/2$, so the next
write is again zero, and the update multiplier is negative.
Since $c\mu w_\eta+\epsilon=\eta\mu/2$,
\[
    |w_k|-w_\eta
    =(y-w_\eta)\frac{\epsilon-c\mu y}{\epsilon+c\mu y}.
\]
The constant $q=\max_{y\in J}|\epsilon-c\mu y|/(\epsilon+c\mu y)$
is less than one. Hence the amplitudes remain in $J$, and induction gives
$\bigl||w_k|-w_\eta\bigr|\leq q^k|w_0-w_\eta|$.
The signs alternate, so $w_{2k}\to w_\eta$ and
$w_{2k+1}\to-w_\eta$. This proves local attraction of the two-cycle.

For Update-RTN, the first working moment satisfies
$\tilde v_1=(c\mu w_0)^2>\tau_\epsilon$, so the first write is $b$.
Every subsequent write also equals $b$, since a preceding write of $b$
gives $\tilde v_k\geq\beta_2b>b/2>\tau_\epsilon$.
The two rules share the same first iterate $w_1$.
Because $\eta\mu<2(\sqrt{\beta_2b}+\epsilon)$, the contraction factor
\[
    \rho:=\max_{0\leq y\leq|w_1|}
    \left|1-\frac{\eta\mu}
    {\sqrt{\beta_2b+(c\mu y)^2}+\epsilon}\right|
\]
is less than one by continuity on the compact interval. Induction yields
$|w_k|\leq\rho^{k-1}|w_1|$ for $k\geq1$, proving geometric
convergence to zero from the same interval $I$.
\end{proof}

\paragraph{Remark before proof of Proposition~\ref{prop:quadratic-rounding-space}\textup{(b)}.}
A zero write does not affect the current update, but it removes the historical second-moment term from the next
recurrence. If the state and the next squared gradient shrink together into the zero-adjacent cell, the resulting
zero-endpoint read grows while the positive-endpoint read remains bounded. A learning rate that decays with the state
scale can therefore leave the zero branch expansive while making the positive branch weakly contractive. State-SR
selects the zero branch with probability approaching one, whereas \upsrmethod{} makes it sufficiently rare for the
positive-branch contraction to dominate.

\vspace{0.08in}
\begin{proof}[Proof of Proposition~\ref{prop:quadratic-rounding-space}\textup{(b)}.]
For the quadratic objective, $g_{t+1}=\mu w_t$, and $\beta_1=0$ makes the working first moment equal to this gradient. The first update influenced by a quantized value $z\in\{0,b\}$ is therefore
\[
    w_{t+1}
    =
    \left(
        1-\eta_{t+1}\mu r_{t+1}(z)
    \right)w_t.
\]
For either rounding rule, let $\Delta_t$ denote the conditional
logarithmic drift appearing in the proposition.
The locally tracked regime and the definition of $r_{t+1}$ give
\[
\begin{aligned}
    r_{t+1}(0)=\Theta\big({\tilde{v}_t}^{-1/2}\big),\qquad
    r_{t+1}(b)=\Theta(1),\qquad
    \frac{r_{t+1}(b)}{r_{t+1}(0)}
    =\Theta\big(\sqrt{\tilde{v}_t}\big).
\end{aligned}
\]
The assumed scaling of the zero endpoint implies
$\eta_{t+1}=\Theta(\sqrt{\tilde{v}_t})\to0$. Consequently,
\begin{align*}
    \log\left|
        1-\eta_{t+1}\mu r_{t+1}(0)
    \right|
    &\longrightarrow
    \log(\alpha-1),\\
    \log\left|
        1-\eta_{t+1}\mu r_{t+1}(b)
    \right|
    &=
    -\Theta\big(\sqrt{\tilde{v}_t}\big).
\end{align*}
State-SR selects zero with probability $1-\tilde{v}_t/b$ and $b$ with
probability $\tilde{v}_t/b$. Therefore
\begin{align*}
    \Delta_t
    &=\left(1-\frac{\tilde{v}_t}{b}\right)
    \log\left|
        1-\eta_{t+1}\mu r_{t+1}(0)
    \right| + \frac{\tilde{v}_t}{b} \,
    \log\left|
        1-\eta_{t+1}\mu r_{t+1}(b)
    \right|\\
    &\longrightarrow
    \log(\alpha-1)>0.
\end{align*}
For \upsrmethod{}, the probability of selecting zero is
\[
    p_0
    =
    \frac{
        \epsilon(1-\sqrt{\tilde{v}_t/b})
    }{
        \sqrt{\tilde{v}_t/(1-\beta_2^t)}+\epsilon
    }
    =\smallO\big(\sqrt{\tilde{v}_t}\big),
\]
where the final relation follows from
$\epsilon=\smallO(\tilde{v}_t)$. Hence
\begin{align*}
    \Delta_t
    =p_0
    \log\left|
        1-\eta_{t+1}\mu r_{t+1}(0)
    \right|
    +\left(1-p_0\right)
    \log\left|
        1-\eta_{t+1}\mu r_{t+1}(b)
    \right|
    =-\Theta\big(\sqrt{\tilde{v}_t}\big),
\end{align*}
which is negative for all sufficiently small $\tilde{v}_t$.
\end{proof}

\end{document}